\def\year{2018}\relax
\documentclass[letterpaper]{article} 
\usepackage{aaai18}  
\usepackage{times}  
\usepackage{helvet}  
\usepackage{courier}  
\usepackage{color}
\usepackage[toc,page]{appendix}
\usepackage{url}  
\usepackage{graphicx}  
\usepackage{epsfig}
\usepackage{mathrsfs} 
\usepackage{mathtools}
\usepackage{amsthm}
\usepackage{footnote}
\makesavenoteenv{tabular}
\makesavenoteenv{table}
\usepackage[font={small,it}]{caption}
\usepackage{subcaption}

\newtheorem{theorem}{Theorem}

\newtheorem{lemma}[theorem]{Lemma}

\def \RR {\mathbf{R}}

\def \bb {\mathbf{B}}
\def \WW {\mathbf{W}}

\newcommand{\pname}[1]{DeepRebirth}

\begin{document}
%
\title{DeepRebirth: Accelerating Deep Neural Network Execution on Mobile Devices}

\author{Dawei Li$^1$, Xiaolong Wang$^1$\thanks{Corresponding author.}, Deguang Kong\\
$^1$Samsung Research America, Mountain View, CA\\
{\tt\small \{dawei.l, xiaolong.w\}@samsung.com}, \{doogkong\}@gmail.com}


\maketitle
\begin{abstract}
Deploying deep neural networks on mobile devices is a challenging task. Current model compression methods such as matrix decomposition effectively reduce the deployed model size, but still cannot satisfy real-time processing requirement. This paper first discovers that the major obstacle is the excessive execution time of non-tensor layers such as pooling and normalization without tensor-like trainable parameters.  This motivates us to design a novel acceleration framework: \pname{} through {\color{black}{ \lq\lq slimming'' existing consecutive and parallel non-tensor and tensor layers.} } {\color{black}{The layer slimming is executed at different substructures: (a) streamline slimming by merging the consecutive non-tensor and tensor layer vertically}}; (b) {\color{black}{branch slimming by merging non-tensor and tensor branches horizontally.}}   The proposed {\color{black}{optimization operations}} significantly accelerate the model execution and also greatly reduce the run-time memory cost since the slimmed model architecture contains less hidden layers. To maximally avoid accuracy loss, the parameters in new generated layers are learned with layer-wise fine-tuning based on both theoretical analysis and empirical verification. As observed in {\color{black}{the}} experiment, \pname{}  achieves more than 3x speed-up and 2.5x run-time memory saving on GoogLeNet with only 0.4\% drop of top-5 accuracy on ImageNet. Furthermore, by combining with other model compression techniques,  \pname{}  offers an average of $65 ms$ inference time on the CPU of Samsung Galaxy S6 with 86.5\% top-5 accuracy, 14\% faster than SqueezeNet which only has a top-5 accuracy of 80.5\%. 
\end{abstract}

\begin{figure}[t]
	\centering
	\epsfig{file=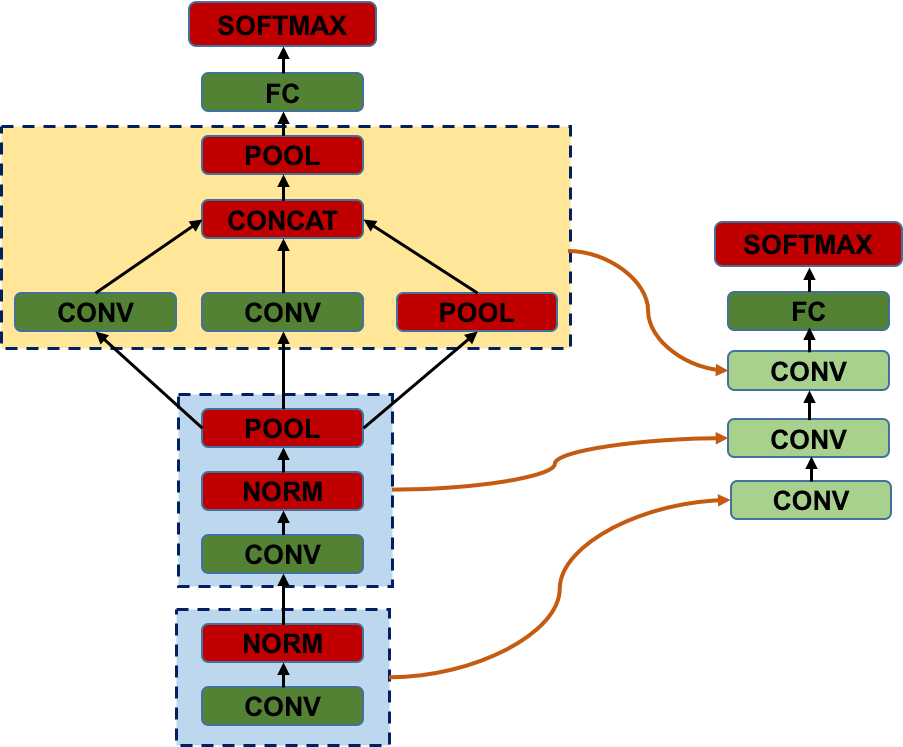, width=0.4\textwidth}
	\caption{An illustration of proposed \pname{} model acceleration {\color{black}{pipeline}}. \pname{} optimizes a trained deep learning model (left) to an accelerated ``slim" model (right). Such optimization is achieved with two {\color{black}{operations}}: \textit{Streamline Slimming} which absorbs non-tensor layers (i.e., pooling and normalization) to their bottom convolutional layer (in light blue background) and \textit{Branch Slimming} which {\color{black}{absorbs}} non-tensor branches and convolutional branches with small convolution filters (e.g., 1x1) to a convolutional branch with large convolution filter (e.g., 5x5) (in light yellow background). We name new {\color{black}{generated}} layers as slim layers.}
	\label{fig:deepslim}
\end{figure}

\section{Introduction}

Recent years have witnessed the breakthrough of deep learning techniques for many computer vision tasks, e.g., image classification~\cite{alexnet,googlenet}, object detection and tracking~\cite{fasterrcnn,POI,du_detection}, video understanding~\cite{donahue2015long,Li_2017_ICCV}, content generation~\cite{goodfellow2014generative,zhang2017age}, disease diagnosis~\cite{shen2017deep,DBLP:journals/corr/abs-1709-00042} and privacy image analytics~\cite{DBLP:conf/aaai/TranKJ016}. 
More and more mobile applications {\color{black}{adopt deep learning techniques}} to provide accurate, intelligent and effective services.  
However, the execution speed of deep learning models on mobile devices becomes a bottleneck for deployment of many applications due to limited computing resources. 

In this paper, we focus on improving the execution efficiency of deep learning models on mobile devices, which is a highly intriguing feature. Here we define the execution efficiency as the model inference speed, the energy cost and the run-time memory consumption. 
In reality, it takes more than 651ms to recognize an image using GoogleNet on Samsung S5 (Table 4) with 33.2 MB run-time memory and 984mJ energy costs (Table 5). The effective solution is expected to provide minimum accuracy loss by leveraging widely used deep neural network architectures (such as GoogLeNet and ResNet) with support of {\color{black}{deep model acceleration}} on different types of layers.


\begin{figure*}[!t]
	\centering
	\begin{subfigure}{.32\textwidth}
		\centering
		\includegraphics[width=0.9\linewidth]{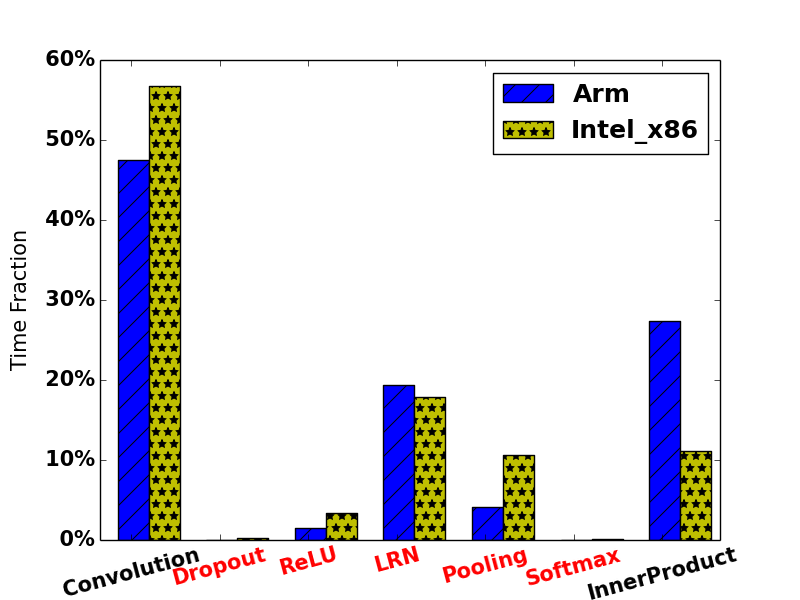}
		\caption{AlexNet}
		\label{fig:alex_fine}
	\end{subfigure}%
	\begin{subfigure}{.32\textwidth}
		\centering
		\includegraphics[width=0.9\linewidth]{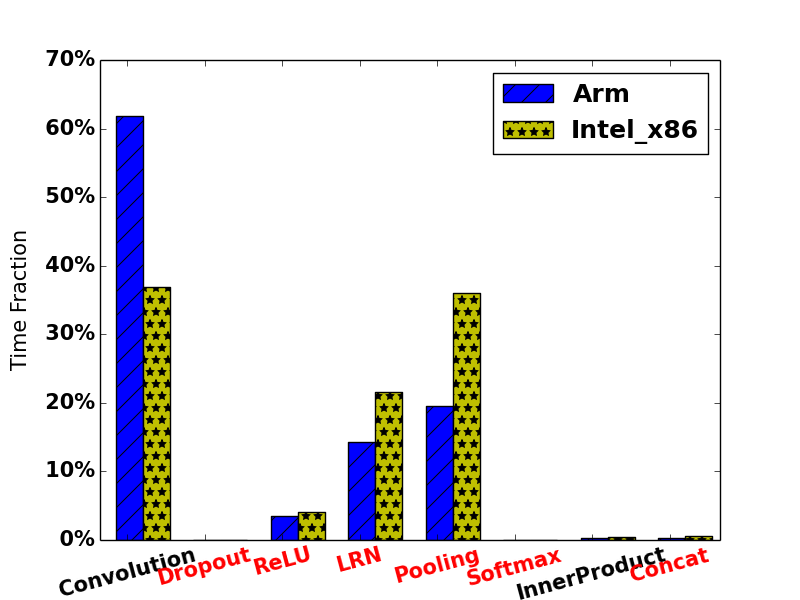}
		\caption{GoogLeNet}
		\label{fig:google_fine}
	\end{subfigure}%
	\begin{subfigure}{.32\textwidth}
		\centering
		\includegraphics[width=0.9\linewidth]{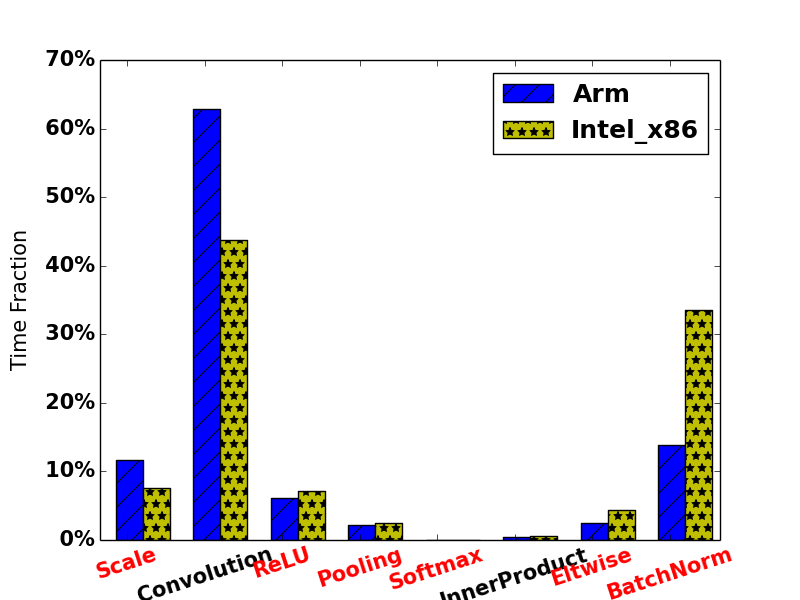}
		\caption{ResNet-50}
		\label{fig:resnet50_fine}
	\end{subfigure}
	\caption{Time Decomposition for each layer. Non-tensor layers (e.g., dropout, ReLU, LRN, softmax, pooling, etc) shown in red color while tensor layers (e.g., convolution, inner-product) shown in black color.}
	\label{fig:layers_fine}
\end{figure*}

{\bf Excessive execution time in Non-tensor layers}

In this work, we find that {\color{black}{non-tensor}} layers consume too much time in model execution (shown in Fig.~\ref{fig:layers_fine}) where \emph{tensor layer} and \emph{non-tensor layer}
are defined based on whether the layer contains tensor-type parameters. For example,  
fully connected layers and convolutional layers are tensor layers since they contain 2-d and 4-d tensor-type weight parameters, respectively. Whereas pooling layer and LRN layer are both non-tensor layers because they do not contain any high-order tensor-type weight parameters. Motivated by this, 
this paper proposes \pname{}, a new deep learning model acceleration framework that significantly reduces the execution time on non-tensor layers. In particular, 
we paid our efforts in two directions: (a) \emph{streamline slimming}; (b) \emph{branch slimming}. In streamline slimming, the new tensor layers are re-generated by {\color{black}{substituting}} the original non-tensor layers and their neighborhood tensor layers in the feed-forward model (shown in Figure \ref{fig:layer_merging}), while in branch slimming, the newly generated tensor layers are created by fusing non-tensor {\color{black}{branches}} with their parallel tensor branches horizontally (shown in Figure \ref{fig:branch_merging}, such as the inception module in GoogLeNet~\cite{googlenet}). Overall, reducing the execution time on non-tensor layers can greatly reduce the model inference time given the fact that tensor-layer has been able to get optimized to the minimum as suggested by ~\cite{DeepCompression,kimtucker}. Finally, we can combine both non-tensor and tensor layer optimization and further reduce {\color{black}{the}} latency as well as the model size.  

\begin{table}
	\caption{Compare \pname{} with Existing Acceleration Methods on CPU of Samsung Galaxy S5 Mobile Device.}
	\label{compare}
	\begin{center}
		\resizebox{0.46\textwidth}{!}{%
			\begin{tabular}{ccccc}
				\multicolumn{1}{c}{\bf }  &\multicolumn{1}{c}{\bf 
					\begin{tabular}{@{}c@{}}Parameter Compression\footnote{The accuracy reported here is based on an compression rate of roughly 50\%. In the original paper, the authors reported a small 0.24\% accuracy loss with compressed rate 31.9\%. For our model at the same 31.9\% compression rate, we also only have a small 0.31\% accuracy loss. } \\ \cite{kimtucker}\end{tabular}
					
				} &\multicolumn{1}{c}{\bf 
					\begin{tabular}{@{}c@{}}SqueezeNet \\ \cite{SqueezeNet}\end{tabular}
				} &\multicolumn{1}{c}{\bf 
					\begin{tabular}{@{}c@{}}MobileNet\footnote{We use the Caffe implementation of 0.5 MobileNet-224 which has similar speed with our model.}  \\ \cite{mobilenets}\end{tabular}
				} &\multicolumn{1}{c}{ \bf
					\begin{tabular}{@{}c@{}}\pname{} \\ (ours)\end{tabular}
				}\\ 
				\hline \\
				Accuracy        & 85.7\%  & 80.3\% & 83.7\% & \textbf{86.5\%}\\
				Execution Time    & 558.3 ms  & 122.7 ms & 109.5 ms & \textbf{106.3 ms}\\
				Energy Cost   & 902 mJ  & 288 mJ & 243 mJ & \textbf{226 mJ}\\
				Memory Cost   & 35.8 MB  & 36.5 MB  & 22.6 MB & \textbf{14.8 MB}\\
		\end{tabular}}
	\end{center}
\end{table}

{\bf Difference with existing works} The central idea of  \pname{} is based on the acceleration of non-tensor layers because
\emph{non-tensor layers are {\color{black}{major obstacles}} for real-time mobile CPU execution (\S \ref{sec:non-tensor}). }   
Compared to existing works,   ~\cite{DeepCompression,kimtucker,cvpr2017} are designed to  reduce the model size by approximating the tensor-type layers using methods like low rank approximation and quantization.  For non-tensor layers (e.g., normalization and pooling layers) which are generally designed and used for speeding up the network training and obtaining better generalization performance, optimization for faster execution has \emph{not} been discussed so far.  In this paper, we emphasize and validate experimentally that the proposed method is orthogonal to compression techniques on tensor-type layers. Consequently,
our method can be combined with these techniques for further acceleration.

To summarize, we make the following contributions: 

$\bullet$  \pname{} is the first work that identifies {\color{black}{the}} excessive execution time of non-tensor layers is the major obstacle for real-time deep model processing on mobile devices.

$\bullet$  \pname{} is also the first work that focuses on optimizing non-tensor layers and significantly accelerates a deep learning model on mobile devices while reducing the required runtime-memory with less layers.


$\bullet$ \pname{} performs both streamline slimming and branch slimming by merging non-tensor layers with its neighboring tensor layers vertically and horizontally,  where the new generated tensor layer parameters are re-trained in a principled way that achieves the same functionality as the original layers.   

$\bullet$  \pname{}  obtained the state-of-the-art speeding up on popular deep learning models with negligible accuracy loss, which enables GoogLeNet to achieve 3x-5x speed-up for processing a single image with only 0.4\% drop on Top-5 accuracy on ImageNet without any weights compression method. \pname{}  achieves around $106.3$ ms for processing a single image with Top-5 accuracy  {\color{black}{up to}} 86.5\%.

\section{Non-tensor layer execution latency}
\label{sec:non-tensor}

\begin{table}[t]
	\caption{Percentage of Forwarding Time on Non-tensor Layers}
	\label{wf-impact}
	\begin{center}
		\resizebox{0.34\textwidth}{!}{%
			\begin{tabular}{cccc}
				\multicolumn{1}{c}{\bf Network}  &\multicolumn{1}{c}{\bf Intel x86} &\multicolumn{1}{c}{\bf Arm} &\multicolumn{1}{c}{\bf Titan X}\\ 
				\hline \\
				AlexNet         & 32.08\%  & 25.08\% & 22.37\%\\
				GoogLeNet    & 62.03\%  & 37.81\% & 26.14\%\\
				ResNet-50   & 55.66\%  & 36.61\% & 47.87\%\\
				ResNet-152    & 49.77\%  & N/A        & 44.49\%\\
				\textbf{Average}   & 49.89\%  & 33.17\%        & 35.22\%\\
		\end{tabular}}
	\end{center}
\end{table}

To give a better understanding of the deep learning model execution latency, we evaluate the execution time cost of different types of layers within a given network structure on several major processors ({\color{black}{Intel}} x86 CPU, Arm CPU and Titan X GPU) using state-of-the-art network structures including AlexNet (Figure \ref{fig:alex_fine}, ~\cite{alexnet}), GoogLeNet(Figure \ref{fig:google_fine}, ~\cite{googlenet}) and  ResNet(Figure \ref{fig:resnet50_fine}, ~\cite{resnet}). 

We define ``percentage non-tensor layer latency" (denoted as $\%$ Latency) as the {\color{black}{time ratio spent}} on non-tensor layers across the {\color{black}{whole}} network, {\it i.e.,} 
\begin{eqnarray}
	\text{ \% Latency} =  \frac{ \text{Time {\color{black}{spent}} on Non-tensor layer}}{ \text{Time {\color{black}{spent}} over the entire network}}, 
\end{eqnarray} 
where larger value indicates the larger execution time cost. 

{\bf Observations and Insights}  The results are shown in Figure \ref{fig:layers_fine} and Table~\ref{wf-impact}. {\color{black}{We can see}}, for classical deep models (e.g., AlexNet), among {\color{black}{these}} non-tensor layers, ``LRN" and ``Pooling" layers are {\color{black}{major}} obstacles that slow-down the model execution. ResNet-50 has abandoned the ``LRN" layers by introducing the \textit{batch normalization} layer, but the findings remain valid as it takes up more than 25\% of the time on ARM CPU and more than 40\% on Intel x86 CPU (in Caffe~\cite{caffe}, it was decomposed into a ``BatchNorm" layer followed by a ``Scale" layer as shown in Figure \ref{fig:resnet50_fine}).  The time fraction spent over such layers ranges from 22.37\% to 62.03\%. Among different types of processors, non-tensor layers have the largest impact on Intel x86 CPUs, and more specifically 62.03\% of the computing time. On the other hand, although non-tensor layers do not have as high affect on the mainstream ARM CPUs, on average they still cost about 1/3 of the computing time.  Therefore, 
\emph{there is a great potential to  accelerate models by optimizing non-tensor layers.}

\section{\pname{}}

To reduce the inference time on non-tensor layers, we propose \pname{} to accelerate the model execution at both streamline substructure and branching substructure. The idea of our method is to merge these highly correlated layers and substitute them as a new ``slim'' layer from the analysis and modeling of the correlations of the current layer and preceding layers (or parallel layers).  As in general deep learning models, the probability distribution of the dataset can be represented by these large redundant tensor layers. 
This process is similar to 
viewing the Inception model as a logical culmination as suggested by \cite{AroraBGM13}. 
\pname{} covers two major components: (a) streamline slimming; (b) branch slimming; which will be illustrated in the following.

\subsection{Streamline Slimming}

For deep network architecture with streamline layer connections,  in order to accelerate the execution, we first identify the layers which have large latency and redundancy. The slimming design is motivated by the key observations: 

$\bullet$ Non-tensor layers usually follow a tensor layer such as convolution layer as shown in Figure \ref{fig:layer_merging}. 

$\bullet$ Several consecutive layers can be viewed as a black box for non-linear transformations, and therefore this can be replaced by a new tensor-layer by parameter learning  to simulate the functionality of original several layers (Figure~\ref{fig:layer_merging}).

{\bf Method}  The streamline slimming regenerates a new tensor layer (i.e., slim layer) by merging non-tensor layers with its bottom tensor units in the feed-forward structure. After layer-wise regeneration, we retrain the deep neural network model by fine-tuning the parameters of the new generated layers. There are two types of streamline slimming in the proposed scheme. The choice of operation depends on the type of non-tensor layers. 

$\bullet$ {\textit{Pooling Layer}:} The pooling layer down-samples feature maps learned from previous layers. Therefore, to absorb a pooling layer to a convolution layer, we remove the pooling layer and set the stride value of the new convolution layer as the product of the stride values for both the original pooling layer and the convolution layer. With a larger stride value for the new slim layer, it further reduces the computation required for executing the new model.

$\bullet$ {\textit{Non-Pooling Layer}:} For non-pooling layers such as LRN and batch normalization, we directly prune those layers from the original deep neural network. 

{\bf Example} Figure \ref{fig:layer_merging}  illustrates how the streamline slimming works.  This is one representative part in GoogLeNet  where the convolution layer $conv2/3\times3$ is followed by a LRN layer $conv2/norm2$ and a pooling layer $poo2/3\times3\_s2$ (The ReLU layer with negligible latency is retained to keep accuracy). Before processing,  the 2 non-tensor layers without a single learned parameter weight take even more time than running the convolution layer. After slimming,  we generate a new slim convolution layer $conv2/3\times3\_merge$, the time spent on the new layer is greatly reduced compare to original layers.

\begin{figure}
	\centering
	\epsfig{file=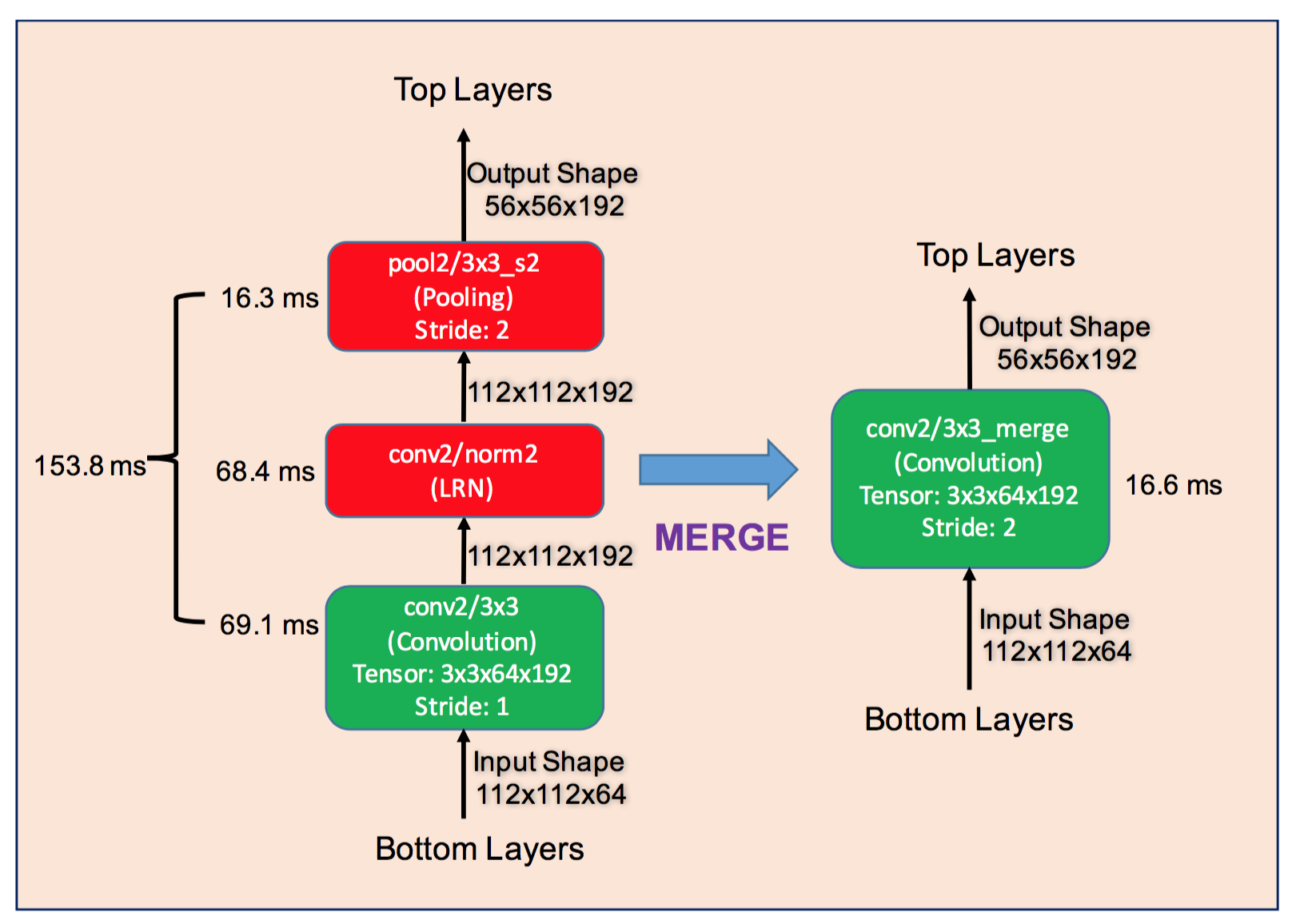, width=0.46\textwidth}
	\caption{Streamline Slimming: The GoogLeNet example and the running time is measured using bvlc\_googlenet model in Caffe on a Samsung Galaxy S5. Left panel:  convolution (in green), LRN (in red), pooling (in red). Right Panel: single convolution layer. The three layers in the left panel are merged and regenerated as a convolution layer (i.e., slim layer) in the right panel.  }
	\label{fig:layer_merging}
\end{figure}

{\bf Theoretical analysis} Given the input image $X^i$, after  several tensor and non-tensor layers, we can get the output feature map   $Y_{\text{CNN}}^i$.
More mathematically,   
\begin{eqnarray}
	\label{EQ:f_map}
	X^i  \xrightarrow{f_{\text{conv}}}  Y^i_{\text{cv}}   \xrightarrow{f_{\text{bn}}}  Y^i_{\text{cv+bn}} \xrightarrow{f_{\text{sl}}}  Y^i_{\text{cv+bn+sl}} \nonumber \\
	\xrightarrow{f_{\text{pooling}}} Y^i_{\text{cv+bn+sl+pl}} \xrightarrow ......:=  Y_{\text{CNN}}^i
\end{eqnarray}
where $f_{\text{conv}}$, ${f_{\text{bn}}} $,  $f_{\text{sl}}$, and $f_{\text{pooling}}$ denote convolution layer, batch normalization layer, scaling layer and pooling layer respectively. There could be other types of layers in the pipeline such as LRN layer $f_{\text{LRN}}$.  
The layer parameters are represented by:  
\begin{eqnarray}
	\label{EQ:conv_bn_sl}
	\;\;\;\;
	\begin{cases}
		f_{\text{conv}}:   \WW_{\text{conv}}, \bb_{\text{conv}};   \\ 
		f_{\text{bn}}:  m, \;\;  \bf{\mu}, \;\;  \bf{\sigma^2};   \\ 
		f_{\text{sl}}:   \bf{\gamma}, \;\;  \bf{\beta};  \\ 
		f_{\text{pooling}}:   p; \\ 
		f_{\text{LRN}}:   \kappa, \;\; \rho \;\; \alpha.  \\ 
		\cdots 
	\end{cases}
\end{eqnarray}
where  $\WW_{\text{conv}}$,  $\bb_{\text{conv}}$ represent convolution layer weight and  bias matrix respectively,  
$\bf{\mu}$,   $\bf{\sigma^2}$ and $m$ are mean, variance,  and sample number in mini-batch of normalization layer  $f_{\text{bn}}$, 
{\bf $\gamma$} and  {\bf $\beta$} are scaling weight and bias in scaling layer $f_{\text{sl}}$ respectively, 
$p$ represents the nearby $p$ regions in pooling layer $f_{\text{pooling}}$, 
and $\kappa$, $\rho$ and $\alpha$ are consecutive feature channel parameters and normalization parameters in LRN layer $f_{\text{LRN}}$.

To achieve the desired functionality with acceleration, the idea is to find a new mapping function $$ \tilde{f}(\tilde{\WW},  \tilde{\bb} )     :\;\;  X^i   \rightarrow  Y_{\text{CNN}}^i, 
$$ such that it can get the same feature map value $ Y_{\text{CNN}}^i$ given the same input feature map $X^i$ for any image $i$.   
Note that operations in Eq.(\ref{EQ:f_map}) transform the feature maps using convolution operations before changing the distributions of activations to avoid ``Internal covariate shift'' in batch normalization ~\cite{DBLP:conf/icml/IoffeS15} at min-batch level, which can be viewed as a new 	``scaling convolution" which transforms the input features in the fully connected layers, and therefore  we  build a single unique convolution operation that replaces several non-tensor layers by setting the new optimization goal,  {\it i.e., }   
\begin{eqnarray}
	\label{EQ:y_i = Y_com}
	\tilde{f}(\tilde{\WW},  \tilde{\bb} )  =:  \tilde{f_{\text{conv}}}(\tilde{\WW_{\text{conv}}},  \tilde{\bb_{\text{conv}}} ) ;  
\end{eqnarray}
Clearly, the optimal solution is given by: 
\begin{eqnarray}
	\label{EQ:w_b}
	(\tilde{\WW^*}, \tilde{\bb^*} ) = argmin_{\WW, \bb}  \sum_i \|  Y_{\text{CNN}}^i  -    \tilde{f}( \WW,  \bb;  X^i )    \|_F^2.
\end{eqnarray}
More formally, we have  lemma ~\ref{lemma1}.
\begin{lemma}
	\label{lemma1}
	Given the input/output feature map pairs ($X^i, Y^i$)  $\forall i$,  operations on the convolution layers followed by non-tensor layers ({\it e.g.}, normalization layer in  Eq.~\ref{EQ:conv_bn_sl}) can be re-trained by learning the new convolution layer  $ \tilde{f}(\tilde{\WW},  \tilde{\bb}) $ via Eq.(\ref{EQ:w_b}) using SGD. 
\end{lemma}
The proof is obvious and therefore we skip it here. In particular, we have  lemma ~\ref{lemma_tmp}. 
\begin{lemma}
	\label{lemma_tmp}
	Let $W_j$, $B_j$,  $\mu_j$,  $\sigma^2_j$,  $\gamma_j$ and $b_j$ be the corresponding  $j$-th dimension in the reshaped weight vector or bias vector in Eq.(\ref{EQ:conv_bn_sl}),  and $\tilde{W_j}$, $\tilde{B_j}$ be the learned new convolution layer parameter in Eq.(\ref{EQ:w_b}).  Then, 
	if $Y_{\text{CNN}}^i$ is obtained after the three layers of $f_{\text{conv}}$, ${f_{\text{bn}}} $,  $f_{\text{sl}}$ in the sequence order, {\it i.e.,} 
	$Y_{\text{CNN}}^i := Y^i_{\text{cv+bn+sl}}, $
	we have closed form solution for the parameters in the new convolution layer:  
	\begin{eqnarray}
		\label{EQ:W_B_S}
		\begin{split}
			& \tilde{W_j} =  \eta_j  W_j,  \\
			& \tilde{B_j} = \eta_j B_j  + \beta_j - \eta_j  \frac{\mu_j}{m},  \\
			& \eta_j = \frac{\gamma_j}{ \sqrt{\frac{\sigma^2_j}{m}}}. 
		\end{split}
	\end{eqnarray}
\end{lemma}
\begin{proof}
	Let $Y_j$ be the $j$-th dimension in feature map after convolution operation in Eqs.(\ref{EQ:y_i = Y_com}, \ref{EQ:w_b}), {\it i.e., $Y_j =  \Big(Y_{\text{CNN}}^i\Big)_j$} . On one hand,  based on the definition of convolution operations (denoted as $*$), we have
	\begin{eqnarray}
		\label{EQ:tY_j = tW}
		Y_j = (\tilde W * X)_j  + \tilde{B_j} .
	\end{eqnarray}
	On the other hand, according to the definition of batch normalization~\cite{DBLP:conf/icml/IoffeS15} and scaling, we have 
	\begin{eqnarray}
		\label{EQ:Y_batch_scala_conv}
		\begin{split}
			Y_j &=   \gamma_j  \Big(f_{\text{bn}} \cdot f_{\text{conv}}(X) \Big)_j +  \beta_j,    \;\;\;     \triangleright \text{ Scaling} \\ 
			&=   \gamma_j  \Big(   \frac{f_{\text{conv}}(X)_j    -  \mu_j }  {   \sqrt{\sigma^2_j} }  \Big)+  \beta_j,   \;\;\;   \triangleright \text{ BN} \\ 
			& =  \gamma_j  \Big(   \frac{ (W * X) _j + B_j    -  \frac{\mu_j} {m} }  {   \sqrt{\frac{\sigma^2_j}{m}} }  \Big)+  \beta_j.   \;\;\;   \triangleright \text{ Convolution} 
		\end{split}
	\end{eqnarray}
	Let $ \eta_j = \frac{\gamma_j}{ \sqrt{\frac{\sigma^2_j}{m}}}$, then Eq.(\ref{EQ:Y_batch_scala_conv}) is equivalent to: 
	\begin{eqnarray}
		\label{EQ:Y_j=}
		\begin{split}
			Y_j =   \underbrace{ \eta_j  (W * X)_j }_{weight}  + \underbrace{ \Big( \eta_j B_j  -      \frac{ \eta_j \mu_j} {m}    +   \beta_j \Big) }_{bias}.   
		\end{split}
	\end{eqnarray}
	Compared to Eq.(\ref{EQ:tY_j = tW}), we have 
	$ \tilde{W_j} =  \eta_j  W_j $ and
	$ \tilde{B_j} = \eta_j B_j  + \beta_j - \eta_j  \frac{\mu_j}{m}.$
	This completes the proof. 
\end{proof}

\subsection{Branch Slimming}

\begin{figure*}
	\centering
	\epsfig{file=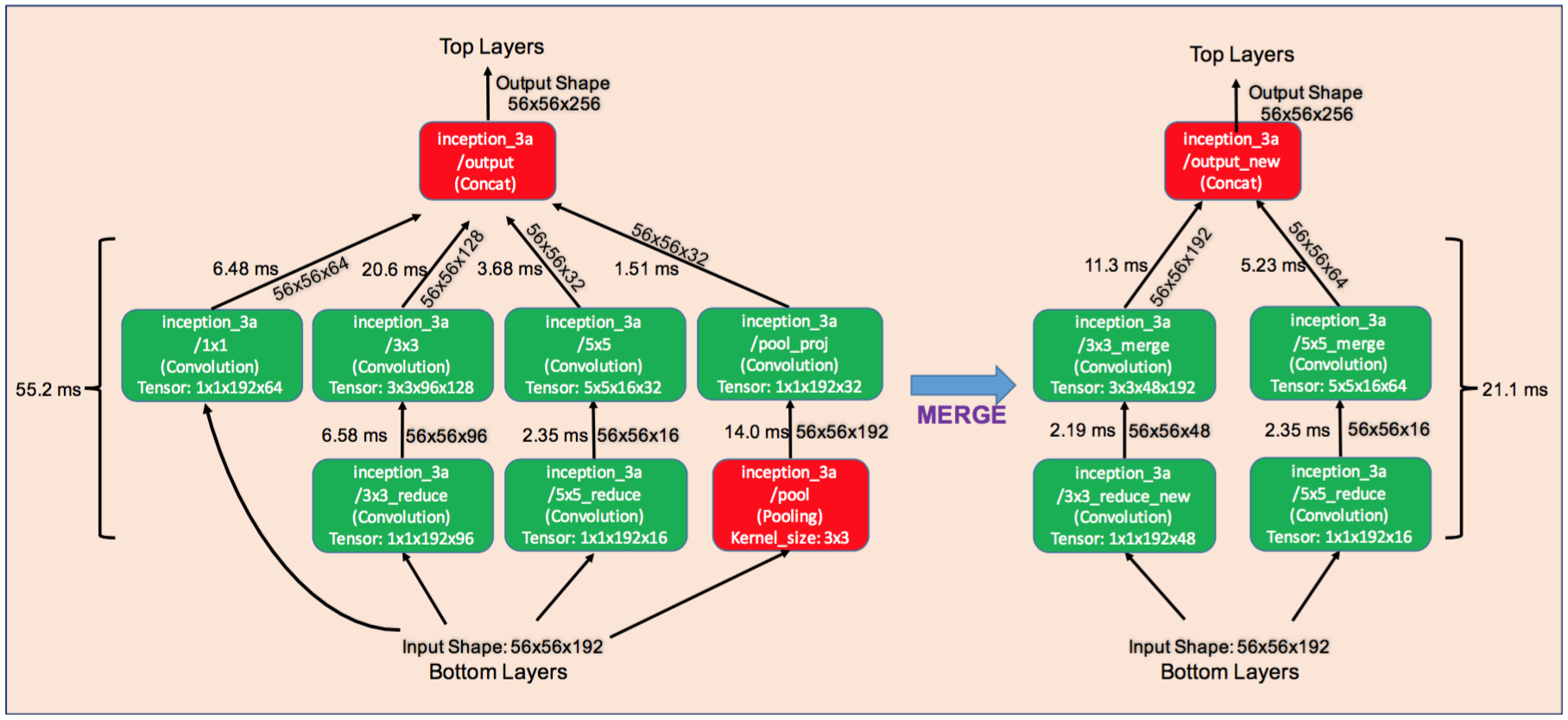, width=0.9\textwidth}
	\caption{Branch Slimming: The GoogLeNet example and the running time is measured using bvlc\_googlenet model in Caffe on a Samsung Galaxy S5. Left panel: four branches in parallel,  convolution layer, convolution + convolution,  convolution + convolution, convolution + pooling. Right panel: two branches in parallel,  convolution + convolution, convolution + convolution. Two branches are reduced. }
	\label{fig:branch_merging}
\end{figure*}

Given the fact that non-tensor layers require more time on computation, if we can learn new tensor layers by fusing non-tensor layers with the tensor units at the same level, then the the execution time will be decreased.  Then we have the deign of \emph{branch slimming}.

{\bf Example}  
One representative unit is the inception module in GoogLeNet. For example, in Figure \ref{fig:branch_merging}, layer ``inception\_3a" of GoogLeNet has 4 branches: 3 convolution branches take feature maps from the bottom layer at various scales ($1\times1$, $3\times3$ and $5\times5$) and one $3\times3$ pooling branch \cite{googlenet}. The output feature maps of each branch are concatenated as input of the following layer.

{\bf Method} For deep network architecture with parallel branches, the output of each branch constitutes part of the feature maps as the input for the next layer. We identify non-tensor branches that have large latency (e.g., the pooling branch in Figure \ref{fig:branch_merging}). Similar to streamline slimming, if we can use a faster tensor branch to simulate the function of the non-tensor branch by relearning its parameters, we can achieve clear speed-up.

To absorb a non-tensor branch into a tensor branch, we re-create a new tensor layer (i.e., slim layer) by fusing the non-tensor branch and a tensor unit with relatively small latency to output the feature maps that were originally generated by the non-tensor branch. If the non-tensor branch has a kernel size larger than $1\times1$ (e.g., the $3\times3$ pooling branch in Figure \ref{fig:branch_merging}), the picked tensor branch's kernel size should be at least the size of the non-tensor branch. As shown in Figure \ref{fig:branch_merging}, we re-learn a new tensor layer ``inception\_3a" by merging the $3\times3$ pooling branch with the $5\times5$ convolution branch at the same level, and the number of feature maps obtained by the $5\times5$ convolution is increased from 32 to 64.

$\bullet$ \textit{Branch Reducing}:
Current deep neural networks usually include convolution branches with $1\times1$ convolution layers (e.g., inception\_3a/3x3\_reduce in Figure \ref{fig:branch_merging}) aiming to reduce feature maps channels. This unit will be processed by a following convolution layer with larger kernel size. For greater speed-up, we further reduce the number of feature maps generated by the $1\times1$ ``reducer". For layer inception\_3a/3x3\_reduce, we reduce the number of output feature maps from 96 to 48. 

$\bullet$ \textit{Tensor-Branch Slimming}:
A convolution branch with a smaller kernel size can be absorbed to a convolution branch with a larger kernel size. The method is similar to the slimming of non-tensor branches. To keep other layers' structures in network unchanged, we remove the small-kernel convolution branch and increase the number of feature maps generated by the large-kernel convolution layers. For examples, for layer inception\_3a/3x3\_reduce, we remove the $1\times1$ convolution branch and increase the number of feature maps generated by the $3\times3$ convolution from 128 to 196. 

Slimming over tensor-branches should be careful. In our work, we demonstrate that in GoogLeNet architecture, the tensor-branch with smaller convolutional kernel can be slimmed without affecting the performance, and thus we are able to reduce 4 branches (3 tensor branches and 1 non-tensor branch) into 2 tensor branches. However, when the original architecture only has 2 tensor branches (e.g., in ResNet), slimming any branch will affect the performance. 

{\bf Branch convolutional layer slimming analysis} 
Let  $Y_{\text{L}}$ and $Y_{\text{R}}$  be the feature map learned using convolution layers respectively given model parameter weight and bias, {\it i.e.,} 
\begin{eqnarray}
	\label{EQ:left+right}
	\begin{cases}
		Y^i_{\text{L}} = {\WW}_{\text{L}}  *  X^i   + {\bb}_{\text{L}};     \;\;\;      \triangleright  \text{    left branch} \\   
		Y^i_{\text{R}} = {\WW}_{\text{R}}  *  X^i  + {\bb}_{\text{R}};     \;\;\;       \triangleright \text{    right branch} 
	\end{cases}
\end{eqnarray}
Let  $Y^i_{\text{L}}$ be the concatenation of feature maps in left and right branches. We wish to learn a new convolution function  $\hat{f}(\WW_{\text{LR}}, \bb_{\text{LR}})$, such that
\begin{eqnarray}
	\label{EQ:LR}
	Y_{LR}^i =[ Y_{\text{left}}^i;  Y_{\text{right}}^i ],   \;\;\; 
	Y_{LR}^i = {\WW}_{\text{LR}} *  X^i   + {\bb}_{\text{LR}}, 
\end{eqnarray}
with $ Y^{L}_i \in \RR^{ M' \times N' \times K'_L} $ and $ Y^{R}_i \in \RR^{ M' \times N' \times K'_R} $ having the same kernel size. 

If ${\WW}_{\text{L}} $ and  ${\WW}_{\text{R}}$ have the same kernel size,  we can get
\begin{eqnarray}
	\label{EQ:W_b}
	{\WW}_{\text{LR}} = [{\WW}_{\text{left}} ;  {\WW}_{\text{right}} ] ,     \;\;\;
	{\bb}_{\text{LR}} = [{\bb}_{\text{left}} ;  {\bb}_{\text{right}} ] .   
\end{eqnarray} 
by substituting Eq.(\ref{EQ:left+right}) into Eq.(\ref{EQ:LR}).  
Otherwise, we need to adjust $Y_{\text{L}}$ and $Y_{\text{R}}$ to the same size and learn the model 
parameters by minimizing:  
\begin{eqnarray}
	\label{EQ:wb-dif}
	(\hat{\WW^*}_{\text{LR}}, \hat{\bb^*}_{\text{LR}} ) = argmin_{\hat{\WW},\hat{\bb}} \sum_i \|  Y_{\text{LR}}^i  -  (\hat{\WW} * X^i + \hat{\bb}) \|_F^2. \nonumber
\end{eqnarray}

\subsection{Adapting \pname{} to Overall Pipeline } 

\pname{} can be easily applied to a pre-trained deep learning model as modern deep model architectures are well-structured with repeating substructures such as the inception module in GoogLeNet and the residual module in ResNet. Generally, there are three golden rules we need to follow: (1) identify the repeating substructures, (2) determine the input dimension and output dimension for each substructure, and (3) apply either streamline slimming or branch slimming based on the substructure type. 

To reconcile the new learned layer with other parts of model, one further step is to fine-tuning the model parameters\footnote{One exception is the BatchNorm layer which can be directly merged to a preceding convolutional layer using Eq.(\ref{EQ:Y_j=})}, as suggested 
in ~\cite{finetune1,finetune2}. In \pname{},  we leverage Xavier \cite{xavier} initialization to initialize the parameters in the new layer while keeping the weights of other layers unchanged. In the optimization procedure, we set the learning rate of new layers 10 times over those in other layers empirically.  Generally, the proposed optimization scheme is applied from the bottom layer to the top layer. Another alternative is to learn multiple slim layers at the same time (we merge and fine-tune 3 sequential inception layers 4b-4d together for GoogLeNet) or merge layers in sequential orders other than from bottom to top. We will explore this discussion in our future work.

\section{Evaluation}
To evaluate the performance of \pname{}, we performed the comprehensive evaluation on top of GoogLeNet,  AlexNet and ResNet. Our implementation is based on Caffe~\cite{caffe} deep learning framework, and we compile it using Android NDK for mobile evaluation. OpenBLAS~\cite{xianyi2014openblas} is used for efficient linear algebra calculations.

\subsection{GoogLeNet}

We use  Caffe's GoogLeNet implementation (i.e., bvlc\_googlenet) with its pre-trained weights. Then we apply the proposed \pname{} optimization scheme to accelerate the running speed of GoogLeNet, which is denoted as ``GoogLeNet-Slim''. 
After non-tensor layer optimization (streamline and branch slimming), we further apply tucker decomposition approach~\cite{kimtucker} to reduce the model size (i.e., the number of learned weights) by 50\%, represented as ``GoogLeNet-Slim-Tucker''.
In addition, we directly employ tucker decomposition method to compress original GoogLeNet. This is indicated as  ``GoogLeNet-Tucker''. Thus, we have 4 variations of GoogLeNet to compare, namely GoogLeNet, GoogLeNet-Slim, GoogLeNet-Tucker and GoogLeNet-Slim-Tucker. We also compare with SqueezeNet \cite{SqueezeNet}, a state-of-the-art compact neural network which includes only 1.2M learnable parameters (vs. 5M for GoogLeNet).

\begin{table}
	\caption{GoogLeNet Accuracy on Slimming Each Layer}
	\label{layer-accuracy}
	\begin{center}
		\resizebox{0.46\textwidth}{!}{%
			\begin{tabular}{cccc}
				\multicolumn{1}{c}{\bf Step}  &\multicolumn{1}{c}{\bf Slim Layer(s)}&\multicolumn{1}{c}{\bf Top-1 Accuracy} &\multicolumn{1}{c}{\bf Top-5 Accuracy} \\ 
				\hline \\
				0    	& N/A  &  68.72\% & 88.89\%  \\
				1    	& conv1  & 68.65\%  & 88.73\%  \\
				2  	& conv2  & 68.66\% & 88.82\% \\
				3    	& inception\_3a  & 68.35\% & 88.50\%  \\
				4    	& inception\_3b  & 68.21\% & 88.27\%  \\
				5  	& inception\_4a  & 68.34\% & 88.60\%  \\
				6 		& inception\_4b-4d  & 68.31\% &  88.61\%   \\
				7  	& inception\_4e  & 68.26\% &  88.43\%   \\
				8    	& inception\_5a & 68.22\% &  88.41\% \\
				9    	& inception\_5b  & 68.03\% & \textbf{88.43\%}   \\
				\textbf{Tucker Decomposition}   &  \textbf{ALL} & 66.71\%  & \textbf{86.54\%}   \\
		\end{tabular}}
	\end{center}
\end{table}

{\bf \underline{Accuracy}} We evaluate the accuracy loss in contrast to original ones after performing the accelerated models. 
The accuracy changing along with the optimization steps conducted on ImageNet ILSVRC-2012 validation dataset are listed in Table \ref{layer-accuracy}. During the whole optimization procedure of model training,  we set the base learning rate for the re-generated layer as 0.01 (the rest layers are 0.001). We apply stochastic gradient descent training method~\cite{sgd} to learn the parameters with a batch size of 32. During our training phase, we set 40,000 as the step size together with 0.1 for gamma value and 0.9 for momentum parameter. At each step, the model generally converges at around 90,000 iterations (2 epochs).

The result indicates that \pname{}  has almost negligible impact on the model accuracy, and the accuracy even increases at certain step (e.g., step 5). This indicates that ``the new-born"  layers perfectly simulate the functionality of previous non-tensor layers before optimization. By applying tucker decomposition method on the slim model to reduce the weights by half (GoogLeNet-Slim-Tucker), we observe that there is a larger drop on accuracy (around 2\%). However, directly applying tucker decomposition method (GoogLeNet-Tucker) to reduce the GoogLeNet weights to a half drops the top-5 accuracy to 85.7\%. 
These results 
imply that our method 
performs reasonable well even after streamline and branch slimming.    

{\bf \underline{Speed-Up}} To evaluate and compare the latency of different optimization approaches, we evaluate the layer-wise running speed on a Samsung Galaxy S5 smartphone with Caffe. Each test run includes 50 subtests with a random input and we report the best test run in terms of forwarding time. During the whole experiment, we turn on the airplane mode and close all other apps. As demonstrated 
in Table \ref{layer-speed},  we observe that GoogLeNet-Slim is 3x faster than GoogLeNet. In addition, as pointed \cite{kimtucker}, the original GoogLeNet model has too many small layers and this results in performance fluctuation. In the worst scenario, GoogLeNet takes around 950 ms for a single forwarding while with reduced number of layers, GoogLeNet-Slim takes only up to 250 ms, which is almost 4x speed-up. The Tucker Decomposition method further reduces the computation for around 50\% at the cost of around 2\% accuracy loss. On the other hand, directly applying tucker decomposition on tensor layers doesn't show any significant acceleration. 

\begin{table}
	\caption{Layer breakdown of GoogLeNet forwarding time cost}
	\label{layer-speed}
	\begin{center}
		\resizebox{0.46\textwidth}{!}{%
			\begin{tabular}{ccccc}
				\multicolumn{1}{c}{\bf Layer} &\multicolumn{1}{c}{\bf GoogLeNet} &\multicolumn{1}{c}{\bf \begin{tabular}{@{}c@{}}GoogLeNet \\ -Tucker\end{tabular}} &\multicolumn{1}{c}{\bf \begin{tabular}{@{}c@{}}GoogLeNet \\ -Slim (ours)\end{tabular}} &\multicolumn{1}{c}{\bf \begin{tabular}{@{}c@{}}GoogLeNet \\ -Slim-Tucker (ours)\end{tabular}}\\ 
				\hline \\
				conv1         & 94.92 ms  		& 87.85 ms  & 8.424 ms  & 6.038 ms \\
				conv2         & 153.8 ms  		&	179.4 ms		& 16.62 ms  & 9.259 ms\\
				inception\_3a   & 55.23 ms  	& 85.62 ms			& 21.17 ms & 9.459 ms \\
				inception\_3b    & 98.41 ms  	& 66.51 ms			& 25.94 ms  & 11.74 ms      \\
				inception\_4a    & 30.53 ms  	& 36.91 ms			& 16.80 ms   & 8.966 ms     \\
				inception\_4b   & 32.60 ms  	& 41.82 ms			& 20.29 ms   & 11.65 ms    \\
				inception\_4c  & 46.96 ms  	& 30.46 ms			&  18.71 ms    & 9.102 ms   \\
				inception\_4d  & 36.88 ms    & 21.05 ms				&  24.67 ms    & 10.05 ms   \\
				inception\_4e    & 48.24 ms  	& 32.19 ms		& 28.08 ms    & 14.08 ms  \\
				inception\_5a    & 24.64 ms  	& 14.43 ms			& 10.69 ms    & 5.36 ms    \\
				inception\_5b    & 24.92 ms  	& 15.87 ms			& 14.58 ms    & 6.65 ms    \\
				loss3    & 				 3.014 ms  	 & 2.81 ms			&  2.97 ms     & 2.902 ms  \\
				\textbf{Total}   & \textbf{651.4 ms} & \textbf{614.9 ms (1.06x)}  & \textbf{210.6 ms (3.09x)}   & \textbf{106.3 ms (6.13x)}      \\
		\end{tabular}}
	\end{center}
\end{table}

We evaluate the speed-up on other popular processors besides Galaxy S5, including (1) Moto E: a low-end mobile ARM CPU, (2) Samsung Galaxy S6: a high-end mobile ARM CPU, (3) Macbook Pro: an Intel x86 CPU, and (4) Titan X: a powerful server GPU.  We demonstrate the experimental results in Table \ref{device-speed} and observe significant speed-up on various types of CPUs. Even on the low-end mobile CPU (i.e., Moto E), around 200 ms model forwarding time is achieved by combining tensor weights compression method. Finally, comparing the proposed approach with SqueezeNet~\cite{SqueezeNet}, we are very excited to see that our optimization approach can obtain faster speed on all mobile devices with much higher accuracy (the Top-5 accuracy for SqueezeNet is 80\%) as listed in Table \ref{device-speed}. 
%
\begin{table}
	\caption{Execution time using different methods (including SqueezeNet) on different processors}
	\label{device-speed}
	\begin{center}
		\resizebox{0.46\textwidth}{!}{%
			\begin{tabular}{cccccc}
				\multicolumn{1}{c}{\bf Device} &\multicolumn{1}{c}{\bf GoogLeNet} &\multicolumn{1}{c}{\bf \begin{tabular}{@{}c@{}}GoogLeNet \\ -Tucker\end{tabular}} &\multicolumn{1}{c}{\bf \begin{tabular}{@{}c@{}}GoogLeNet \\ -Slim\end{tabular}} &\multicolumn{1}{c}{\bf \begin{tabular}{@{}c@{}}GoogLeNet \\ -Slim-Tucker\end{tabular}} &\multicolumn{1}{c}{\bf SqueezeNet}\\ 
				\hline \\
				Moto E   					 &	  1168.8 ms & 897.9 ms  & 406.7 ms  & \textbf{213.3 ms} & 291.4 ms\\
				Samsung Galaxy S5   &  651.4 ms & 614.9 ms   & 210.6 ms  & \textbf{106.3 ms} & 136.3 ms\\
				Samsung Galaxy S6   & 424.7 ms  & 342.5 ms & 107.7 ms  & \textbf{65.34  ms} &  75.34 ms \\
				Macbook Pro (CPU)     & 91.77 ms & 78.22 ms & 23.69 ms  & \textbf{15.18 ms } &  17.63 ms  \\
				Titan X    						&  10.17 ms & 10.74 ms  & 6.57 ms   & 7.68 ms   & \textbf{3.29 ms}  \\
		\end{tabular}}
	\end{center}
\end{table}

\begin{table}
	\caption{Storage, Energy and Runtime-Memory Comparison}
	\label{storage-memory}
	\begin{center}
		\resizebox{0.46\textwidth}{!}{%
			\begin{tabular}{ccccc}
				\multicolumn{1}{c}{\bf Model}  &\multicolumn{1}{c}{\bf Energy} &\multicolumn{1}{c}{\bf Storage} &\multicolumn{1}{c}{\bf Memory} &\multicolumn{1}{c}{\bf \begin{tabular}{@{}c@{}}Max Batch Size \\ on Titan X\end{tabular}} \\ 
				\hline \\
				GoogLeNet       		& 984 mJ & 26.72 MB & 33.2 MB & 350 \\
				GoogLeNet-Tucker   & 902 mJ   & 14.38 MB & 35.8 MB & 323\\
				GoogLeNet-Slim   & \textbf{447 mJ (2.2x)}  & 23.77 MB & 13.2 MB &  \textbf{882 (2.52x)}\\
				GoogLeNet-Slim-Tucker  & \textbf{226 mJ (4.4x)} & 11.99 MB & 14.8 MB & \textbf{785 (2.24x)}\\
				SqueezeNet  & 288 mJ & 4.72 MB & 36.5 MB & 321 \\
		\end{tabular}}
	\end{center}
\end{table}

{\bf \underline{Energy, Storage and Runtime-Memory Cost}}
We measure the energy cost of each compared model using PowerTutor Android app on Samsung Galaxy S5 (similar results are obtained on other mobile devices). The original GoogLeNet consumes almost 1 Joule per image while GoogLeNet-Slim consumes only 447 mJ. Applying tucker decomposition further reduces the energy cost to only 1/4 at 226 mJ. When deploying to the mobile devices, we remove the loss1 and loss2 branches from the trained models so that the storage cost of each model is reduced by 24.33 MB. GoogLeNet-Slim which achieves significant speed-up does not save much storage cost compared to the original GoogLeNet model. However, for modern mobile devices, storage is not a scarce resource (e.g., Samsung Galaxy S5 has 16 GB or 32 GB storage), so a 20 MB deep learning model is ``affordable" on mobile devices. Meanwhile, we can always perform the tensor weights compression method to further reduce the storage cost.

Another benefit of layer slimming is run-time memory saving. The generated GoogLeNet-Slim model reduces the number of layers and consumes only 13.2 MB to process one image. This feature is also very useful for the cloud based deep learning service which can process a much larger batch at one run. As shown in table~\ref{storage-memory}, one Titan X GPU can run a batch size of 882 with the GoogLeNet-Slim model while the original GoogLeNet can only allow a batch size of 350. On the other hand, SqueezeNet though has much less trained parameters, it has much larger run-time memory impact due to the increased number of layers.


\subsection{AlexNet and ResNet}
We  apply the proposed framework to other popular deep neural structures: AlexNet~\cite{alexnet} and ResNet~\cite{resnet}. Note that we did not apply tensor weights compression to those two models which can further reduce the model forwarding latency. First, we study the classical AlexNet model. We apply streamline slimming approach to re-generate new slim layers by merging the first two convolution layers followed by LRN layers. We illustrate the result in Table \ref{alex-result}. This indicates that by applying slimming to the first two layers, the model forwarding time of AlexNet is reduced from 445 ms to 274 ms on Samsung Galaxy S5, and the Top-5 accuracy is slightly dropped from 80.03\% to 79.57\%. 

\begin{table}
	\caption{AlexNet Result (Accuracy vs. Speed vs. Energy cost)}
	\label{alex-result}
	\begin{center}
		\resizebox{0.46\textwidth}{!}{%
			\begin{tabular}{ccccc}
				\multicolumn{1}{c}{\bf Step}  &\multicolumn{1}{c}{\bf Slim Layer(s)} &\multicolumn{1}{c}{\bf Top-5 Accuracy} &\multicolumn{1}{c}{\bf Speed-up}&\multicolumn{1}{c}{\bf Energy Cost}\\ 
				\hline \\
				0    	& N/A  & 80.03\%  &  445 ms    &  688 mJ  \\
				1    	& conv1+norm1 $\,\to\,$ conv1  & 79.99\%  & 343 ms (1.29x) &  555 mJ (1.24x) \\
				2  	& conv2+norm2 $\,\to\,$ conv2  & 79.57\%  & 274 ms (1.63x)  &  458 mJ (1.51x) \\
		\end{tabular}}
	\end{center}
\end{table}

\begin{table}
	\caption{ResNet (conv1-res2a) Result (Accuracy vs. Speed up). For each step, we absorb the ``BatchNorm" and ``Scale" layers to the bottom convolution layer. }
	\label{resnet-result}
	\begin{center}
		\resizebox{0.46\textwidth}{!}{%
			\begin{tabular}{ccccc}
				\multicolumn{1}{c}{\bf Step}  &\multicolumn{1}{c}{\bf Slim Layer(s)} &\multicolumn{1}{c}{\bf Top-5 Accuracy} &\multicolumn{1}{c}{\bf Speed-up}&\multicolumn{1}{c}{\bf Runtime-Mem Batch32}\\ 
				\hline \\
				0    	& N/A  & 92.36\%  &  189 ms     & 2505 MB \\
				1    	& conv1  & 92.13\%  &  162 ms (1.17x)   & 2113 MB (1.19x)\\
				2    	& res2a\_branch1  & 92.01\%  & 140 ms (1.35x)   & 1721 MB (1.46x)\\
				3   	& res2a\_branch2a-2c  & 91.88\%  & 104 ms (1.82x) & 1133 MB (2.21x)\\
		\end{tabular}}
	\end{center}
\end{table}

We  apply the acceleration scheme to the more advanced ResNet model. In the experiment, we use the popular 50-layer ResNet-50 model as baseline. We mainly apply the acceleration framework to conv1 and res2a layers (res2a has 2 branches; one branch has 1 convolution layer and another branch has 3 convolution layers).  We present the result in Table \ref{resnet-result}. The time latency on Samsung Galaxy S5 for the processed layers (i.e., conv1 and res2a) is reduced from 189 ms to 104 ms. Moreover, the run-time memory cost is reduced by 2.21x. The accuracy is only slightly reduced. Meanwhile, since batch normalization layers can be directly merged to their preceding convolutional layers using Eq.(\ref{EQ:Y_j=}), additional 30\%-45\% speed-up can be achieved without accuracy loss as indicated by Figure \ref{fig:resnet50_fine}.


\section{Related Work}

Reducing the model size and accelerating the running speed are two general ways to facilitate the deployment of deep learning models on mobile devices. Many efforts have been spent on reducing the model size.  In particular,  most works focus on optimizing tensor-layers to reduce the model size due to the high redundancy in the learned parameters in tensor layers of a given deep model.  Vanhoucke et al. \cite{vanhoucke2011improving} proposed a fixed-point implementation with 8-bit integer activation to reduce the number of parameter used in the deep neural network while  \cite{gong2014compressing} applied vector quantization to  compressed deep convnets.  These approaches, however, mainly focus on  compressing the fully connected layer without considering the convolutional layers. 
To reduce the parameter size,  Denten {\it et al.} \cite{denton2014exploiting} applied the low-rank approximation approach to compress the neural networks with linear structures. Afterwards, hashing functions, which have been widely adopted to improve efficiency of traditional computer vision tasks~\cite{hasing0,du_hashing}, were utilized to reduce model sizes by randomly grouping connection weights~\cite{chen2015compressing}. More recently, Han et al.\cite{DeepCompression} proposed to effectively reduce model size and achieve speed-up by the combination of pruning, Huffman coding and quantization. However, the benefits can only be achieved by running the compressed model on a specialized processor \cite{han2016eie}.

Recently, SqueezeNet~\cite{SqueezeNet} has became widely used for its much smaller memory cost and increased speed. However, the near-AlexNet accuracy is far below the state-of-the-art performance.  Compared with these two newly networks, our approach has much better accuracy with more significant acceleration. Springenberg et al. \cite{DBLP:journals/corr/SpringenbergDBR14} showed that the conv-relu-pool substructure may not be necessary for a neural network architecture. The authors find that max-pooling can simply be replaced by another convolution layer with increased stride without loss in accuracy. Different from this work,  \pname{} replaces a complete substructure (e.g., conv-relu-pool, conv-relu-LRN-pool) with a single convolution layer, and aims to speed-up the model execution on the mobile device. In addition, our work slims a well-trained network by relearning the merged layers and does not require to train from scratch. Essentially, \pname{}  can be considered as a special form of distillation \cite{distillation} that transfers the knowledge from the cumbersome substructure of multiple layers to the new accelerated substructure.

\section{Conclusion and Future Work}

An acceleration framework -- \pname{} is proposed to speed up the neural networks with satisfactory accuracy,  which operates by re-generating new tensor layers from optimizing non-tensor layers and their neighborhood units.  
\pname{}  
is also compatible with  state-of-the-art deep models like GoogleNet and ResNet, where most parameter weight compression methods failed. By applying \pname{} on different deep learning architectures, 
we obtain significant speed-up on different processors (including mobile processors), which will readily facilitate the deployment of deep learning models on mobile devices
in the new AI tide.   

In future work, we plan to integrate \pname{} with other state-of-the-art tensor layer compression methods and also extend our evaluation to heterogeneous mobile processors such as mobile GPUs, DSPs. We envision that understanding the characteristics of these different chips can help us design better algorithms and further improve the model execution efficiency. 

\section{Acknowledgements}

We thank all the anonymous reviewers for their insightful comments and valuable suggestions.

{
\small
\bibliographystyle{aaai}
\bibliography{iclr2017_conference}
}

\end{document}